\documentclass[letterpaper, 10 pt, conference]{ieeeconf}  

\IEEEoverridecommandlockouts                              

\usepackage[utf8]{inputenc} 
\usepackage[T1]{fontenc}    
\usepackage{hyperref}       
\usepackage{url}            
\usepackage{booktabs}       
\usepackage{amsfonts}       
\usepackage{nicefrac}       
\usepackage{microtype}      
\usepackage{amsmath}
\usepackage{amssymb}
\usepackage{algorithm}
\usepackage{algorithmic}
\usepackage{graphicx}
\usepackage{todonotes}

\newtheorem{theorem}{Theorem}

\newcommand*{\QEDA}{\hfill\ensuremath{\blacksquare}}%
\newcommand{\KL}{\mathtt{KL}}
\newcommand{\calL}{\mathcal{L}}

\newcommand{\xbold}{{x}}
\newcommand{\ubold}{{u}}

\newcommand{\goal}{{g}}
\newcommand{\start}{{s}}
\newcommand{\nominal}{{\bar{u}}}

\newcommand{\tron}{\textsc{Tron}}
\newcommand{\ilqr}{\textsc{Ilqr}}
\newcommand{\admm}{\textsc{Admm}}

\newcount\Comments  
\definecolor{darkgreen}{rgb}{0,0.5,0}
\definecolor{darkred}{rgb}{0.7,0,0}
\definecolor{teal}{rgb}{0.3,0.8,0.8}
\definecolor{orange}{rgb}{1.0,0.5,0.0}
\definecolor{purple}{rgb}{0.8,0.0,0.8}
\newcommand{\kibitz}[2]{\ifnum\Comments=1{\textcolor{#1}{\textsf{\footnotesize #2}}}\fi}

\title{\LARGE \bf \tron{}: A Fast Solver for Trajectory Optimization with Non-Smooth Cost Functions}

\author{Anirudh Vemula$^{1}$ and J. Andrew Bagnell$^{2}$
  \thanks{$^{1}$Robotics Institute, Carnegie Mellon University {\tt\small vemula@cmu.edu}}
  \thanks{$^{2}$Aurora Innovation {\tt\small dbagnell@ri.cmu.edu}}
}

\begin{document}

\maketitle

\begin{abstract}
  Trajectory optimization is an important tool for control and
  planning of complex, underactuated robots, and has shown impressive
  results in real world robotic tasks. However, in applications where the cost function
  to be optimized is non-smooth, modern trajectory
  optimization methods have extremely slow convergence.
  In this work, we present \tron{}, an iterative solver that can be used for efficient trajectory
  optimization in applications with non-smooth cost
  functions that are composed of smooth components.
  \tron{} achieves this by exploiting the structure of the objective
  to adaptively
  smooth the cost function, resulting in a sequence of objectives
  that can be efficiently optimized.
  \tron{} is provably guaranteed to converge to
  the global optimum of the non-smooth convex cost function when the dynamics are linear,
  and to a stationary point when the dynamics are nonlinear.
  Empirically, we show that
  \tron{} has faster convergence and lower final costs when compared to
  other trajectory optimization methods on a range of simulated 
  tasks including collision-free motion planning for a
  mobile robot, sparse
  optimal control for surgical needle, and a satellite rendezvous problem.
\end{abstract}

\section{Introduction}
\label{sec:introduction}

Trajectory optimization is a general framework that can be used to
synthesize dynamic motions for robots with complex nonlinear dynamics
by computing feasible state and control sequences that minimize a
cost function while satisfying constraints \cite{trajopt,chomp}. Most of the
existing methods in this framework exploit the
differentiability (or smoothness) properties of the cost function to be
optimized. However, many realistic applications require the use of
cost functions that are not smooth. For example, consider the
task of computing an optimal control sequence for steering an
autonomous car.
A control sequence for the steering that is not sparse is undesirable,
as it results
in steering behavior that does not mimic a human driver who tend to
have sparse controls.
This could deteriorate the driving
experience for the passenger.
Traditionally, sparsity is enforced in
optimization by penalizing the L1-norm \cite{lasso} of the control, which makes the
resulting cost function non-differentiable. Other examples of
non-smooth cost functions include
minimum-fuel \cite{fuel}, and minimum-time \cite{time}
objectives. Thus, there are a broad range of applications in robotics
and other scientific domains which require the use of
non-smooth cost functions.

Unfortunately, modern trajectory optimization methods have extremely slow convergence when
dealing with non-smooth cost functions \cite{Vossen}. Previous work
has tackled this challenge by smoothing the cost function and
optimizing the smoothed objective \cite{chomp,
  van2014iterated}. This results in convergence to a trajectory whose
suboptimality is heavily dependent on the extent of smoothness
introduced. Vossen and Maurer
\cite{Vossen} introduced an approach specific to L1-norm 
objectives using regularization and augmentation techniques, but it is
only applicable to objectives that are linear in the controls. More
recently, Le Cleac’h and Manchester~\cite{lefast} proposed a method
based on ADMM~\cite{DBLP:journals/ftml/BoydPCPE11}
specifically tackling the L1-norm problem that can handle general nonlinear
dynamics and constraints. However their approach does
not exploit the structure of L1-norm objective and exhibits slow
convergence, as shown in our experiments.

We present \tron{}, an iterative solver that is applicable to a broad family of
non-smooth cost functions, and exploits the
structure of the objective to achieve fast convergence.
More specifically, we focus on non-smooth cost functions that are
composed of smooth components.
\tron{} is very easy to implement and requires trivial modifications
to popular trajectory optimization methods such as
\textsc{Ilqr}~\cite{DBLP:conf/icinco/LiT04} and
\textsc{Ddp}~\cite{doi:10.1080/00207176608921369}.
We derive our method by
formulating the optimization problem as a two-player min-max game to construct a
sequence of adaptively smoothed 
objectives that can be efficiently optimized using modern trajectory optimization
methods.
\tron{} is provably guaranteed to
converge to the global optimum of the non-smooth convex
objective when dynamics are linear, and to a stationary point when
dynamics are nonlinear.
We show that \tron{}
exhibits fast convergence when compared with other 
trajectory optimization methods,
on a range of applications including collision-free motion planning
for a mobile robot, sparse optimal control for a surgical needle, and a
satellite rendezvous problem.

We introduce the broad family of
non-smooth cost functions that we consider in this work in
Section~\ref{sec:problem}, and present our min-max optimization
objective and a simple solution strategy in \tron{} in Section~\ref{sec:approach}. We present
a convergence analysis of \tron{} in
Section~\ref{sec:convergence-analysis} and demonstrate how \tron{} can be
applied in the context of trajectory optimization in
Section~\ref{sec:appl-traj-optim}. Finally, our experimental results
demonstrate the effectiveness of \tron{} in  Section~\ref{sec:experiments} and conclude with potential
future extensions in Section~\ref{sec:conclusion}.

\section{Problem Formulation}
\label{sec:problem}

Trajectory optimization solves the following general
problem:
\begin{equation}
  \label{eq:7}
  \tag{A}
\begin{aligned}
  \min_{x_{0:T}, u_{0:T-1}} \quad & \ell_T(x_T) + \sum_{t=0}^{T-1} \ell_t(x_t,
  u_t) \\
  \textrm{subject to} \quad & x_{t+1} = \kappa(x_t, u_t) \\
  \quad & \alpha_t(x_t, u_t) \leq 0 \\
  \quad & \beta_t(x_t, u_t) = 0
\end{aligned}
\end{equation}
where $t$ denotes the time step index, $\ell_T$ and $\ell_t$ denote
the final and $t$-th stage cost functions, $x_t$ and $u_t$ denote the state and
control of the trajectory at time step $t$, $T$ is the horizon, $\kappa(x_t, u_t)$ denotes
discrete dynamics, and $\alpha_t$ and $\beta_t$ denote inequality and equality
constraints on the state and control inputs. For simplicity of exposition, 
we assume there are no constraints on the state
and control inputs except the dynamics $x_{t+1} = \kappa(x_t,
u_t)$\footnote{\tron{} can be extended to account for additional
constraints using, e.g., augmented lagrangian
techniques~\cite{hestenes1969multiplier}.}.

In this work, we assume that the cost functions $\ell_t(x_t, u_t)$ and
$\ell_T(x_T)$ in problem~\eqref{eq:7} have the following structure:
\begin{equation}
  \label{eq:14}
  \begin{aligned}
  &\ell_t(x_t, u_t) = f_t(x_t, u_t) + \sum_{i=1}^M\max\{g_t^i(x_t,
    u_t), \bar{g}_t^i(x_t, u_t)\} \\
  &\ell_T(x_T) = f_T(x_T) + \sum_{i=1}^M\max\{g_T^i(x_T),
  \bar{g}_T^i(x_T)\}
  \end{aligned}
\end{equation}
where the functions $f_t, g_t^i, \bar{g}_t^i$ are continuous,
twice-differentiable and convex functions. Note that the resulting
objective in problem~\eqref{eq:7} may be non-smooth due to the $\max$ terms
in the cost functions $\ell_t$ and $\ell_T$. Thus, the objective to be
optimized is a non-smooth function with smooth
components\footnote{A broad range of applications require cost
  functions that possess this structure. For example, the L1-norm objective adheres to this
  structure since $\|a\|_1 = \sum_{i=1}^n \max(a_i, -a_i)$ where $a =
  [a_1, \cdots, a_n]^T \in \mathbb{R}^n$. More examples in Section~\ref{sec:experiments}}.  \tron{} can also be trivially
extended to the case where there are more than two functions involved
in the $\max$ operator. We discuss several extensions of \tron{} in
Section~\ref{sec:conclusion}.

\section{Iterative Solver for Non-Smooth Objectives}
\label{sec:approach}

Our aim is to solve the optimization problem given in
equation~\eqref{eq:7}. However, for ease of exposition, we
will tackle the general version of problem~\eqref{eq:7} without the
dynamics constraints given by,
\begin{equation}
  \label{eq:1}
  \tag{B}
  \min_{y \in Y} f(y) + \max\{g_1(y), g_2(y)\}
\end{equation}
where $Y = \mathbb{R}^n$ is a closed and convex set, functions $f, g_1, g_2 :
\mathbb{R}^n \rightarrow \mathbb{R}$
are twice-differentiable convex functions in $Y$. Note that this is
simply a general
version of problem~\eqref{eq:7} when combined with the structure
assumed in equation~\ref{eq:14} (See
Section~\ref{sec:appl-traj-optim}). An example of such an objective is
shown in Figure~\ref{fig:hinge}.

\begin{figure}[t]
  \centering
  \includegraphics[width=0.7\linewidth]{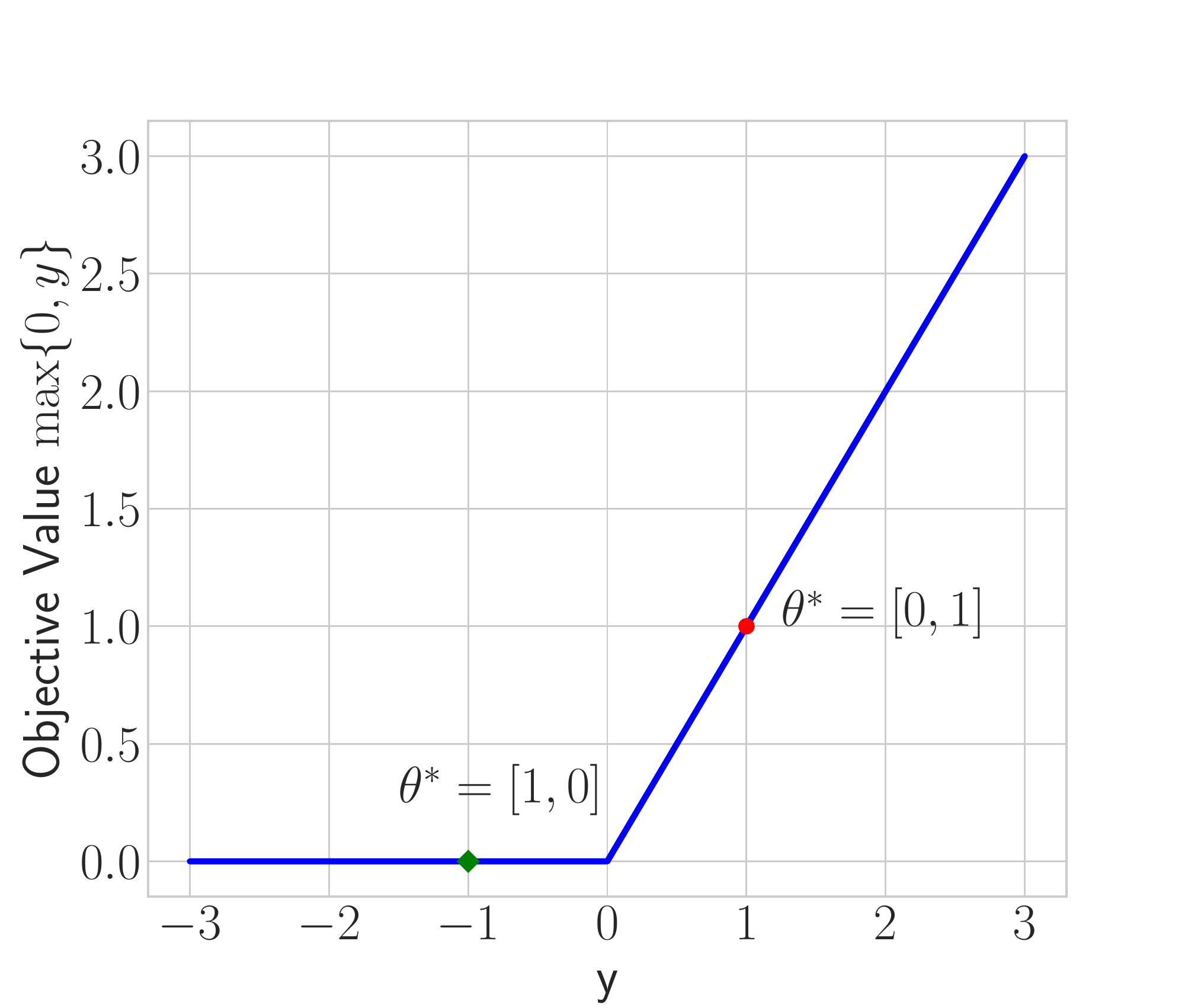}
  \caption{An example of objective in problem~\eqref{eq:1} with $f(y)
    = 0$, $g(y) = 0$ and $\bar{g}(y) = y$. The red point corresponds to
    $y = 1$ and the green point corresponds to $y = -1$.
  }
  \label{fig:hinge}
\end{figure}

We will formulate the above optimization problem in equation \eqref{eq:1}
as a two player min-max game where one player seeks to minimize the following
objective in $y \in Y$ and the other player maximizes the objective in
$\theta = [\theta_1, \theta_2]^T \in \Delta_2$ where $\Delta_2$
denotes the $2$-dimensional
simplex:
\begin{equation}
  \label{eq:2}
  \min_{y \in Y} f(y) + \max_{\theta \in \Delta_2} (\theta_1 g_1(y) +
  \theta_2g_2(y))
\end{equation}


Observe that
the inner maximization objective in
problem~\ref{eq:2} is linear in $\theta \in
\Delta_2$. Hence for any $y \in Y$ if both $g_1(y)$ and $g_2(y)$ are not
zero, the optimal $\theta^*$ will lie on the boundary of the 
simplex~\cite{10.5555/548834}, specifically one of $\theta_1, \theta_2$ should be $1$ and
the other $0$. In the case where both $g_1(y) = 0$ and $g_2(y) =
  0$, problem~\eqref{eq:1} and problem~\ref{eq:2} are trivially equivalent.
Substituting $\theta^*$ in the objective in problem~\ref{eq:2} it
reduces to $ f(y) + \max(g_1(y), g_2(y))$, the objective in
problem~\eqref{eq:1}.
Thus, any solution of problem~\ref{eq:2} is also a solution of the problem~\eqref{eq:1}.

However, this equivalence is not useful since the term $\max_{\theta \in \Delta_2}
(\theta_1g_1(y) + \theta_2g_2(y))$ could be highly non-smooth in $y$, which
results in $\theta^*$ varying drastically
with changing $y$. An example of such behavior is shown in
Figure~\ref{fig:hinge}, where if $y$ is changed between any two values
across $y=0$ then $\theta^*$ oscillates between $[0, 1]$ and $[1, 0]$.
Borrowing insights from online convex optimization
\cite{DBLP:journals/corr/abs-1909-05207}, we stabilize it
by adding a regularization term 
that penalizes deviations from the
previous estimate for $\theta$. More precisely, if we solve problem~\ref{eq:2}
iteratively and at any iteration $k$, we have an estimate
$\theta^{k-1}$ then we seek to optimize the following objective for
the $k$-th iteration
\begin{equation}
  \label{eq:3}
  \min_{y \in Y} f(y) + \max_{\theta \in \Delta_2} (\theta_1g_1(y) +
  \theta_2g_2(y) - \eta^{k}\KL(\theta||\theta^{k}))
\end{equation}
where $\eta^{k} \geq 0$ is the penalty coefficient at iteration $k$ and
$\KL(\theta||\theta^{k})$ is the KL-divergence regularization term that penalizes
deviation of $\theta$ from the previous estimate $\theta^{k}$ given by
\begin{equation}
  \label{eq:4}
  \KL(\theta||\theta^{k}) =
  \theta_1\log\frac{\theta_1}{\theta^{k}_1} + \theta_2\log\frac{\theta_2}{\theta^{k}_2}
\end{equation}
for any $\theta, \theta^{k} \in \Delta_2$.

Note that the objective in problem~\ref{eq:3} is an approximation of
the objective in problem~\ref{eq:2} (and hence, problem~\eqref{eq:1}) and is exact when $\eta^k =
0$.
To rectify this, we will solve the approximate objective iteratively for a
decreasing sequence $\{\eta^k\}$ where $\eta^k
\rightarrow 0$ as $k \rightarrow \infty$. Thus, in the limit the
solution to the approximate objective is the same as the original objective.




It is important to note that the inner maximization w.r.t $\theta$ in
problem~\ref{eq:3} can be \textit{solved in closed form}. Writing the
Lagrangian and solving the KKT conditions~\cite{DBLP:journals/iandc/KivinenW97}, we get the following
iterative update
for $\theta$ at any iteration $k$,
\begin{equation}
  \label{eq:5}
  \theta^k = \frac{\theta^{k-1}\exp({\frac{g(y)}{\eta^k}})}{\sum_{i=1}^2
  \theta^{k-1}_i\exp({\frac{g_i(y)}{\eta^k}})}
\end{equation}
where we denote $g(y) = [g_1(y), g_2(y)]^T \in
\mathbb{R}^2$, and $\theta^k = [\theta_1^k, \theta_2^k]^T \in \Delta_2$.
We can obtain $y^k$ by substituting update~\ref{eq:5} into the
objective in problem~\ref{eq:3} resulting in the following optimization problem at any
iteration $k$,
\begin{equation}
  \label{eq:6}
  \min_{y \in Y} f(y) + \eta^k\log(\theta_1^{k-1}\exp({\frac{g_1(y)}{\eta^k}})
  + \theta_2^{k-1}\exp({\frac{g_2(y)}{\eta^k}}))
\end{equation}
This results in an implicit update for $y$ that accounts for the
$\theta$ update. This is reminiscent of
implicit online learning~\cite{DBLP:conf/icml/KulisB10}, which
typically has faster convergence and is robust in adversarial settings.
The solution $y^k$ obtained from solving
problem~\ref{eq:6} can then be substituted into the update in
problem~\ref{eq:5} to get $\theta^k$. Let us denote the objective in
equation \ref{eq:6} as $\calL_{\eta^k}(y, \theta^{k-1})$. 
It is useful to observe that the objective $\calL_{\eta^k}$ is smooth and
twice-differentiable in $y$.
Thus for each iteration $k$, we obtain a smoothed approximation of the
objective in problem~\eqref{eq:1} and as $\eta^k \rightarrow 0$ we get a
tighter approximation. We would like to emphasize that the proposed
solver \tron{} exploits the structure of problem~\eqref{eq:1} by
restricting $\theta$ to be in the simplex, and by using
KL-divergence as the regularization. This trick has connections to
exponentiated gradient descent~\cite{DBLP:journals/iandc/KivinenW97}, which also uses KL-divergence as
regularization to perform efficient optimization on a simplex. As we will
see in the experiments, this enables \tron{} to quickly solve
problem~\eqref{eq:1}. \tron{} is also related to proximal methods~\cite{DBLP:journals/ftopt/ParikhB14} and
augmented lagrangian methods~\cite{hestenes1969multiplier}.
The proposed iterative
solver is summarized in Algorithm~\ref{alg:abstract}.

\begin{algorithm}[t]
  \caption{\tron{} for general structured non-smooth objectives}
  \label{alg:abstract}
  \begin{algorithmic}[1]
    \STATE {\bfseries Input:} Number of iterations $K$, sequence
    $\{\eta^k\}$ such that $0 \leq \eta^{k+1}\leq \eta^k$ and $\eta^k
    \rightarrow 0$, sequence
    $\{\epsilon^k\}$ such that $\epsilon^k \geq 0$ and $\epsilon^k
    \rightarrow 0$
    \STATE Initialize $\theta^0 \in \Delta_2$, $y^0 \in Y$
    \FOR {$k=1$ to $K$}
    \STATE Solve problem~\ref{eq:6} (warm-starting from $y^{k-1}$) to obtain $y^k$ such that $\|\nabla_y\calL_{\eta^k}(y^k,
    \theta^{k-1})\| \leq \epsilon^k$
    \STATE Obtain $\theta^k$ using update~\ref{eq:5} with $y = y^k$
    \ENDFOR
    \STATE {\bfseries Return:} Solution $y^K$
  \end{algorithmic}
\end{algorithm}




\section{Convergence Analysis}
\label{sec:convergence-analysis}

In this section, we present convergence analysis for \tron{}
described in Algorithm~\ref{alg:abstract}. We would like to show that
as $k \rightarrow \infty$, every limit point of the sequence of
solutions $\{y^k\}$ is a stationary point of the original
problem~\eqref{eq:1}. The following theorem states this guarantee: 
\begin{theorem}[Convergence under Inexact Minimization]
  Assume $Y = \mathbb{R}^n$, and $f, g$ are continuously
  differentiable. For $k=1, \cdots$ let $y^k$ satisfy
  \begin{equation*}
    \|\nabla_y \calL_{\eta^k}(y^k, \theta^{k-1})\| \leq \epsilon^k
  \end{equation*}
  where $\{\theta^k\}$ is bounded, and $\{\epsilon^j\}$ and
  $\{\eta^k\}$ satisfy
  \begin{align*}
    &0 \leq \eta^{k+1} \leq \eta^k, \eta^k \rightarrow 0 \\
    &0 \leq \epsilon^k, \epsilon^k \rightarrow 0
  \end{align*}
  Then every limit point $y^*$ of the sequence $\{y^k\}$ is a stationary
  point of problem~\eqref{eq:1}, i.e. $0 \in
  \partial(f(y) + \max(g_1(y), g_2(y)))$ or $\nabla_yf(y^*) +
  \lambda\nabla_yg_1(y^*) + (1 - \lambda)\nabla_yg_2(y^*) = 0$ for
  some $\lambda \in [0, 1]$.
  \label{theorem:inexact-minimization}
\end{theorem}
\begin{proof}
  Proof given in Appendix~\ref{sec:asympt-conv-under-1}.
\end{proof}

Theorem~\ref{theorem:inexact-minimization} guarantees that \tron{}
converges to a stationary point of
problem~\eqref{eq:1}. In
Section~\ref{sec:appl-traj-optim}, we will show that this implies
convergence to the global minimum of the trajectory optimization
problem~\eqref{eq:7} when the dynamics are linear, and
convergence to a stationary point when the dynamics are nonlinear.

\section{Application to Trajectory Optimization}
\label{sec:appl-traj-optim}

In Section~\ref{sec:approach} we presented \tron{}, an iterative solver that
can be used to efficiently solve problem~\eqref{eq:1}. We will now show how \tron{} can be used to
solve trajectory optimization problem~\eqref{eq:7} when
the cost functions are non-smooth with smooth components. Let us rewrite
problem~\eqref{eq:7} to
accommodate the structure from equation~\ref{eq:14} in the cost
function as follows (using $M=1$ for ease of notation):
\begin{equation}
  \label{eq:15}
  \begin{aligned}
  \min_{x_{0:T}, u_{0:T}} \quad & \sum_{t=0}^{T} f_t(x_t, u_t) +
  \max\{g_t(x_t, u_t), \bar{g}_t(x_t, u_t)\} \\
  \textrm{subject to} \quad & x_{t+1} = \kappa(x_t, u_t)
  \end{aligned}
\end{equation}
where $f_T(x, u) = f_T(x)$, $g_T(x, u) = g_T(x)$, and $\bar{g}_T(x, u) = \bar{g}_T(x)$.
Observe that the above objective is of the same form as the objective
in problem~\eqref{eq:1}. Hence, we can use \tron{} to optimize the above problem. Formulating the above problem as a min-max
game as in Section~\ref{sec:approach}, we get the following objective
in the state-control inputs at iteration $k$,
\begin{equation}
  \label{eq:16}
  \begin{aligned}
    \min_{x_{0:T}, u_{0:T}} & \quad \sum_{t=0}^{T}
    f_t(x_t, u_t) + \eta^k\log(\theta_t^{k-1}\exp({\frac{g_t(x_t,
        u_t)}{\eta^k}}) \\
    &~~~~~~~~~~~~~~~~~~~~~~~~~~~~~+ \bar{\theta}^{k-1}_t\exp({\frac{\bar{g}_t(x_t,
      u_t)}{\eta^k}})) \\
  \textrm{subject to} & \quad  x_{t+1} = \kappa(x_t, u_t)
  \end{aligned}
\end{equation}

Notice that unlike
problem~\ref{eq:6}, we have a dynamics constraint $x_{t+1} =
\kappa(x_t, u_t)$. We account for this by using \textsc{Ilqr}
\cite{DBLP:conf/icinco/LiT04} to optimize the objective
in problem~\ref{eq:16}, thereby implicitly enforcing
the dynamics constraint in our solver\footnote{We can use any
  trajectory optimization solver such as
  \textsc{Ddp}~\cite{doi:10.1080/00207176608921369},
  \textsc{Chomp}~\cite{chomp} in place of \textsc{Ilqr}}. We use the control trajectory
from the previous iteration to warm-start \textsc{Ilqr} at the current
iteration to ensure it remains 
fast. The entire procedure to solve
the trajectory optimization problem~\ref{eq:15} using \tron{} is described in Algorithm~\ref{alg:trajopt}.

\begin{algorithm}[t]
  \caption{Trajectory Optimization with Structured Non-smooth Cost
    Functions using \tron{}}
  \label{alg:trajopt}
  \begin{algorithmic}[1]
    \STATE {\bfseries Input:} Initial state $x_0$, initial control sequence
    $u^0_{0:T-1}$, sequence $\{\eta^k\}$,
    Number of iterations $K$
    \STATE Initialize $[\theta_t^0, \bar{\theta}_t^0]^T \in \Delta_2$ for $t=0,\cdots,T$
    \STATE Compute $x_{0:T}^0 \leftarrow$ Rollout dynamics $\kappa$
    from $x_0$ using $u_{0:T-1}$
    \FOR {$k=1$ to $K$}
    \STATE $x^k_{0:T}, u^k_{0:T-1} \leftarrow$ Solution of \textsc{Ilqr}
    problem~\ref{eq:16}
    \STATE Obtain $[\theta_t^k, \bar{\theta}_t^k]^T \in \Delta_2$ for
    $t=0,\cdots,T$ using update~\ref{eq:5}
    \ENDFOR
    \STATE {\bfseries Return:} $x_{0:T}^K, u_{0:T-1}^{K-1}$
  \end{algorithmic}
\end{algorithm}

We can use Theorem~\ref{theorem:inexact-minimization} to analyze the
convergence properties of Algorithm~\ref{alg:trajopt}.
Observe
that when the dynamics $\kappa$ are linear, problem~\ref{eq:15} is
convex. Using the fact that any stationary point of a convex problem
is a global minimum in conjunction with Theorem~\ref{theorem:inexact-minimization}, we
can guarantee that \tron{} converges to the global minimum of
problem~\ref{eq:15} when dynamics $\kappa$ are linear. However, when the dynamics $\kappa$
are arbitrarily nonlinear, then we lose convexity of
problem~\ref{eq:15} and \tron{} is only guaranteed to converge to a
stationary point.

\section{Experiments}
\label{sec:experiments}

In this section, we will evaluate the empirical performance of \tron{}
against baselines on four tasks: lasso problem with synthetic data\eqref{sec:lasso},
collision-free motion planning for a mobile robot\eqref{sec:traj-optim-diff}, sparse optimal
control for a surgical steerable needle\eqref{sec:sparse-optimal-control}, and a satellite rendezvous
problem\eqref{sec:sparse-optim-contr}. \tron{} and other baselines are implemented in Python on
a $3.1$GHz Intel Core i5 machine, and the code is released at
\url{https://github.com/vvanirudh/TRON}. 

\subsection{Lasso Problem with Synthetic Data}
\label{sec:lasso}

In this experiment, we will solve a $2$D lasso problem given as
follows:
\begin{equation}
  \label{eq:13}
  \min_{w} \frac{1}{N}\|Xw - y\|_2^2 + \rho\|w\|_1
\end{equation}
where $N = 1000, X \in \mathbb{R}^{1000 \times 2}$, $y \in \mathbb{R}^{1000}$ and $\rho \in
\mathbb{R}^+$ are synthetically generated. We use $\rho = 0.05$ for
this experiment. We implemented \tron{} with Newton's method to
solve sub-problems~\ref{eq:6}, and compare it with Newton's method on
the non-smooth problem~\ref{eq:13} and subgradient method. For the
subgradient method, initial learning rate is chosen carefully and is
decayed at $\mathcal{O}(\frac{1}{\sqrt{k}})$ where $k$ is the
iteration number. We implement Newton's method using a backtracking
line search to compute the newton direction. For \tron{}, we use $\eta^{0} = 1$ and
update $\eta^{k+1} = 0.9\eta^k$ for each iteration $k$.


\begin{figure}[t]
  \centering
  \includegraphics[width=0.9\linewidth]{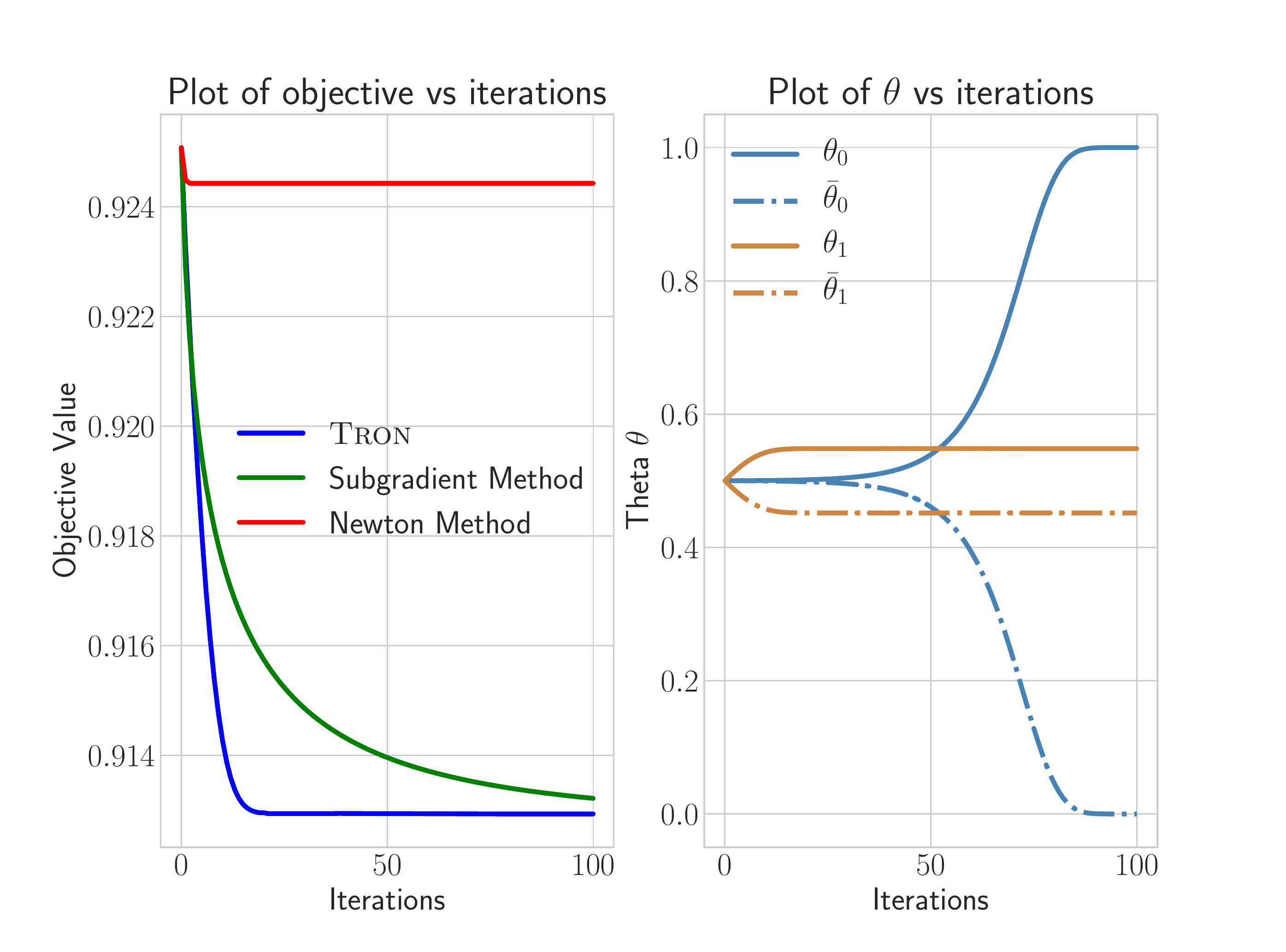}
  \caption{Performance of \tron{}, subgradient method and Newton's
    method on the $2$D lasso problem. On the left, we plot the
    objective value vs iterations of each method. On the right, we
    plot how the dual variables $\theta$ vary across iterations.}
  \label{fig:lasso}
\end{figure}


The results are shown in Figure~\ref{fig:lasso}. As the objective in
problem~\ref{eq:13} is not differentiable, once we approach close to the
minimum Newton's method gets stuck as line search returns extremely
small steps. This results in Newton's method having extremely slow
convergence as shown in Figure~\ref{fig:lasso} (left). Subgradient
method, on the other hand, does not rely on line search and with the
help of decaying learning rate makes steady but slow progress towards
the minimum. \tron{}, using Newton's method to optimize the adaptively
smoothed objective, quickly converges to the global minimum of the
problem. In Figure~\ref{fig:lasso} (right), we plot the dual variables
$\theta$ as they vary across iterations. We start with initial values
of $0.5$ for all the dual variables. The
converged value of $w$ for \tron{} is $[-0.0005, 0]$, and as expected
the corresponding dual variables for each dimension converge to $1, 0$
for the non-zero component and to values in $[0, 1]$ for the zero
component. See the proof of Theorem~\ref{theorem:inexact-minimization}
in Appendix~\ref{sec:asympt-conv-under-1} for theoretical insights on the
final converged value of the dual variables.

\subsection{Collision-Free Motion Planning for a Mobile Robot}
\label{sec:traj-optim-diff}

Our second experiment involves a simulated differential drive mobile
robot. The state is defined by vector $\xbold = [p_x, p_y, \theta]^T
\in \mathbb{R}^3$
where $(p_x, p_y)$ describes the robot's two-dimensional position, 
$\theta$ is its orientation, and the control input is defined by the
vector $\ubold = [v_l, v_r]^T \in \mathbb{R}^2$ where $v_l, v_r$ describes the left and
right wheel speeds (m/s) respectively. The dynamics of the robot are given
by the following equations, $\dot{p_x} = \frac{1}{2} (v_l + v_r)\cos\theta$,
$\dot{p_y} = \frac{1}{2} (v_l + v_r)\sin\theta$, and
$\dot{\theta} = (v_r - v_l) / w$,
where
$w$
is the distance between the wheels of the robot. This setup
is very similar to the experimental setup used in
\cite{DBLP:conf/isrr/Berg13}. We discretize the dynamics using a
third-order Runge Kutta integrator.

The circular robot is moving in an environment
with $O = 11$ obstacles (see Figure~\ref{fig:diffdrive} left) and needs to move from
a specified start state to a goal state while avoiding collision with
obstacles. We use the following cost functions in problem~\eqref{eq:7} to achieve this
objective,
\begin{align*}
  &\ell_0(\xbold_0, \ubold_0) = (\xbold_0 - \start)^TQ(\xbold_0
    - \start) + (\ubold_0 - \nominal)^TR(\ubold_0 - \nominal) \\
  &\ell_t(\xbold_t, \ubold_t) = (\ubold_t -
    \nominal)^TR(\ubold_t - \nominal) + \rho \sum_{i=1}^O \max\{0,
    -\nu d_i(\xbold_t)\} \\
  &\ell_T(\xbold_t) = (\xbold_t - \goal)^TQ(\xbold_t - \goal)
\end{align*}
where $\goal$ is the goal state, $\start$ is the start state, $T$ is
the horizon, and
$\nominal$ is the nominal control input. $\rho, \nu$ are positive
scalar factors, and the function $d_i(\xbold)$ gives the signed
distance between the robot at state $\xbold$ and the $i$-th obstacle
of the environment. Note that we penalize a trajectory if it results
in the robot penetrating an obstacle (thus, $d_i(\xbold)$ is
negative for any $i$), and zero cost if the robot does not collide
with any of the obstacles. Note that this cost
function is non-smooth and convex.

\begin{figure}[t]
  \centering
  \includegraphics[width=0.9\linewidth]{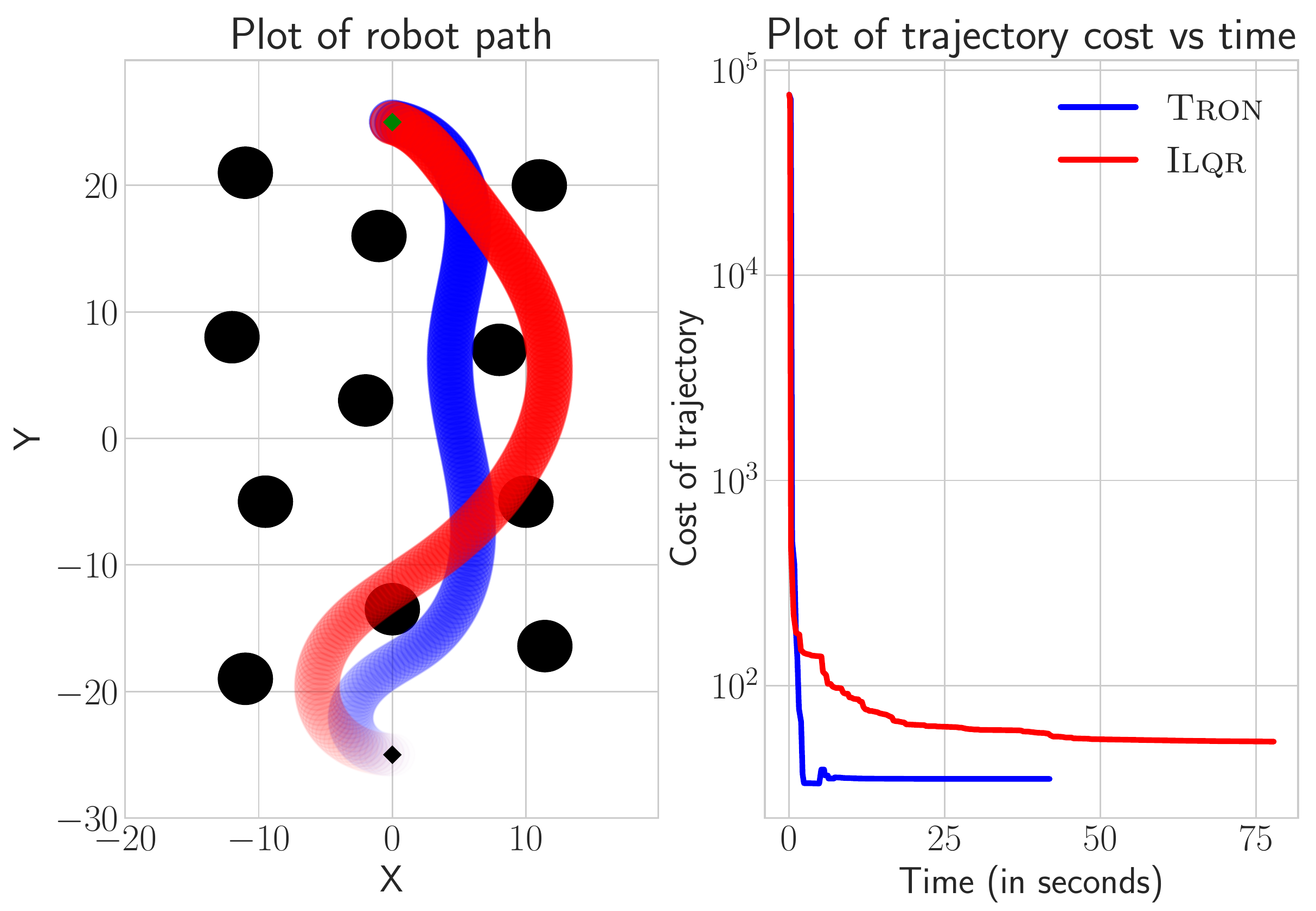}
  \caption{Performance of \tron{} and \textsc{Ilqr} for trajectory
    optimization of a differential drive robot. On the left, the
    resulting trajectory from both \tron{} (blue) and \ilqr{}
    (red) are shown (lighter the color, earlier the time step). The
    start position is depicted using black diamond, and the goal
    position is depicted using green diamond. On the
    right, we plot the cost vs time (in seconds)
    for both methods. Both methods are run for $200$ iterations.}
  \label{fig:diffdrive}
\end{figure}

We compare \tron{} (with a fixed $\eta^k = 1$ for all iterations) with a baseline that uses \ilqr{} on the non-smooth
objective. The results are
shown in Figure~\ref{fig:diffdrive}. \ilqr{} exhibits extremely slow
convergence and does not converge to a collision-free path as shown in
Figure~\ref{fig:diffdrive} (left). On the other hand, \tron{} quickly
converges to a collision-free path with a significantly lower cost
compared to \ilqr{}. Note that the y-axis in
Figure~\ref{fig:diffdrive} (right) is in log-scale. It is also
interesting to observe the path \tron{} converges to. Since the cost
function $\ell_t$ only penalizes if the robot collides with the
obstacle and zero penalty otherwise, we see that the resulting path
narrowly avoids collision with obstacles, and leads the robot directly
to the goal between the obstacles on a low-cost trajectory.

\subsection{Sparse Optimal Control for a Surgical Steerable Needle}
\label{sec:sparse-optimal-control}

Our third experiment involves a simulated bevel-tip surgical steerable needle
that is highly underactuated and non-holonomic \cite{DBLP:conf/icra/DuindamASG08}. Planning the motion of
the needle is a challenging problem as it can only be controlled from
its base through insertion and twisting. We use the motion model
proposed in \cite{DBLP:journals/ijrr/WebsterKCCO06} where the state of
the needle $x = [p_x, p_y, p_z, \alpha, \beta, \gamma]^T \in \mathbb{R}^6$ is represented by a transformation matrix $X \in SE(3)$:
\begin{align*}
  X = 
  \begin{bmatrix}
    R & p \\
    0 & 1
  \end{bmatrix}
\end{align*}
where $R \in SO(3)$ is a $3\times 3$ rotation matrix describing
needle's orientation constructed from $[\alpha, \beta, \gamma]^T$
which is the euler angle representation, and $p = [p_x, p_y, p_z]^T\in \mathbb{R}^3$ describes its
position. The control input $u = [v, w, \delta]^T \in \mathbb{R}^3$ is represented as:
\begin{align*}
  U =
  \begin{bmatrix}
    W & V \\
    0 & 1
  \end{bmatrix}, ~~~~~
        W =
        \begin{bmatrix}
          0 & -w & 0 \\
          w & 0 & -v\delta \\
          0 & v\delta & 0
        \end{bmatrix}
\end{align*}
where $V = [0, 0, v]^T$, $v$ is the linear velocity of
the needle tip (m/s), $w$ is the angular speed of the needle base
(rad/s), and $\delta$ is the desired curvature of the needle. The
kinematics of the needle are given by $\dot{X} = XU$. This setup is
very similar to the experimental setup used in
\cite{DBLP:conf/wafr/BergPAAG10}.

The task is to plan a path for the needle from a fixed start state to a goal
state ensuring kinematic feasibility. In addition, we would also like
sparsity in the angular speed $w$ control as rotation of the needle
inside a body increases trauma to patient tissues. We use very similar
objectives as in Section~\ref{sec:traj-optim-diff} except the
following:
\begin{align*}
  \ell_t(x_t, u_t) = (u_t - \bar{u})^TR(u_t - \bar{u}) + \rho|w|
\end{align*}
where $u = [v, w, \delta]^T$ and $\rho$ is a positive scalar. Thus, we penalize the absolute value (or
L1-norm) of
the angular speed $w$ to enforce sparsity in the resulting trajectory
for that control input. This cost
function is convex but non-smooth.

\begin{figure}[t]
  \centering
  \includegraphics[width=0.9\linewidth]{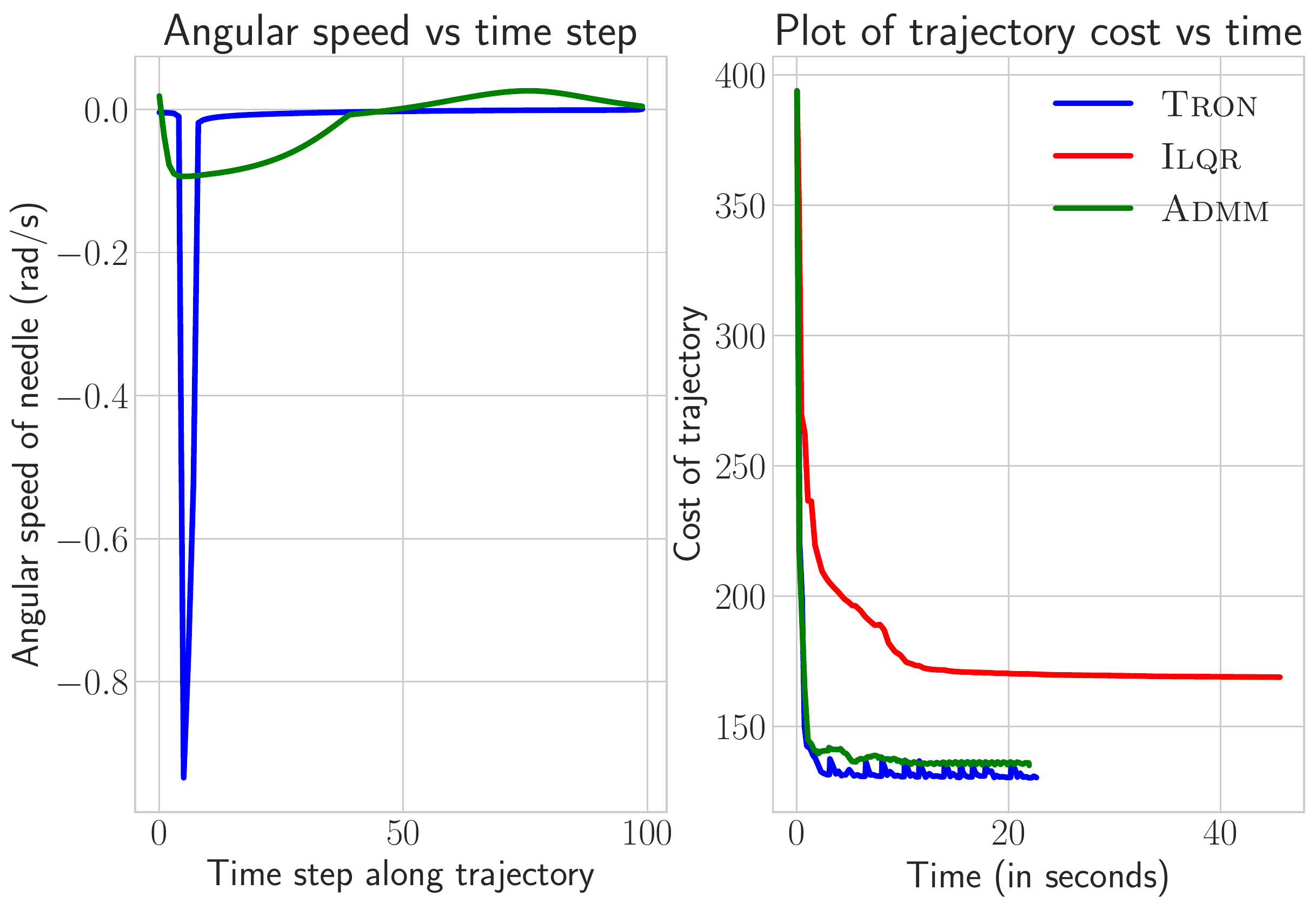}
  \caption{Performance of \tron{}, \admm{}, and \ilqr{} on the
    surgical needle task. On the right, we plot the cost of the
    trajectory vs time (in seconds) for all the methods. On the left, we show the angular speed
    control sequence for the final converged trajectory for \tron{}
    and \admm{} methods. We do not show \ilqr{} as it does not exhibit
  any sparse behavior. All approaches are run for $100$ iterations.}
  \label{fig:needle}
\end{figure}

We compare \tron{} (with a fixed $\eta^k = 0.3$ for all iterations)
with a baseline that uses \ilqr{} on the non-smooth objective. In
addition, we also implement the \admm{} approach proposed by Le
Cleac’h and Manchester \cite{lefast} which accounts for the L1-norm
penalty. The results are shown in Figure~\ref{fig:needle}. As expected
on non-smooth objectives, \ilqr{} exhibits slow convergence and we have
noticed that it does not result in a trajectory that has sparse
angular speed control. \admm{}, on the other hand, accounts for the
L1-norm penalty and shows fast convergence behavior. As shown in
Figure~\ref{fig:needle} (right), it converges to a trajectory that has
significantly lower cost compared to \ilqr{}. \tron{} also exhibits fast convergence
behavior similar to \admm{} and as shown in Figure~\ref{fig:needle}
(left), it does a much better job at enforcing sparsity in angular
speed of the needle in the final trajectory, (in fact, \admm{} does not achieve
any sparsity in its final trajectory) and achieves lower final
trajectory cost.


\subsection{Satellite Rendezvous Problem}
\label{sec:sparse-optim-contr}

Our last experiment involves a simulated satellite rendezvous
problem. These satellites rely on reaction control system thrusters
for control which can only operate inside a limited range, and are
suitable for a bang-off-bang control strategy. Thus, enforcing
sparsity in control is desirable to achieve this strategy. The
objective of this task is to control a chaser satellite so that it
approaches and docks onto a target satellite. We borrow the
linearized version of this problem from \cite{lefast} in which the state
vector $x = [p_1, p_2, p_3, \dot{p}_1, \dot{p}_2, \dot{p}_3] \in
\mathbb{R}^6$, where $[p_1, p_2, p_3]^T$ is the position and
$[\dot{p}_1, \dot{p}_2, \dot{p}_3]^T$ is the velocity, both expressed
in a frame centered on the target satellite. The control input $u$ is
the force applied on the satellite, and the model is given as follows:
\begin{align*}
  \dot{x} =
  \begin{bmatrix}
    \dot{p}_1 \\
    \dot{p}_2 \\
    \dot{p}_3 \\
    3n^2p_1 + 2np_2 + u_1/m \\
    -2n\dot{p}_1 + u_2/m \\
    -n^2p_3 + u_3/m
  \end{bmatrix}
\end{align*}
where $n$ is the mean motion of the target satellite's orbit and $m$
is the satellite's mass. This setup is similar to the setup used in
\cite{lefast} with small modifications.

The objective of the task is for the chaser satellite to reach the
state $[0, 0, 0, 0, 0, 0]^T$ (which is the target satellite's position
and zero velocity) by the end of the trajectory horizon. We use the
following cost functions in problem~\eqref{eq:7} to achieve this
objective,
\begin{align*}
  \ell_t(x_t, u_t) &= \alpha\|u_t\|_1 + u_t^TRu_t\\
  \ell_T(x_T) &= x_T^TQx_T
\end{align*}
for $0 \leq t \leq T-1$. Thus, the objective function enforces that
the chaser satellite reaches the target and uses sparse controls to
achieve it.

\begin{figure}[t]
  \centering
  \includegraphics[width=0.9\linewidth]{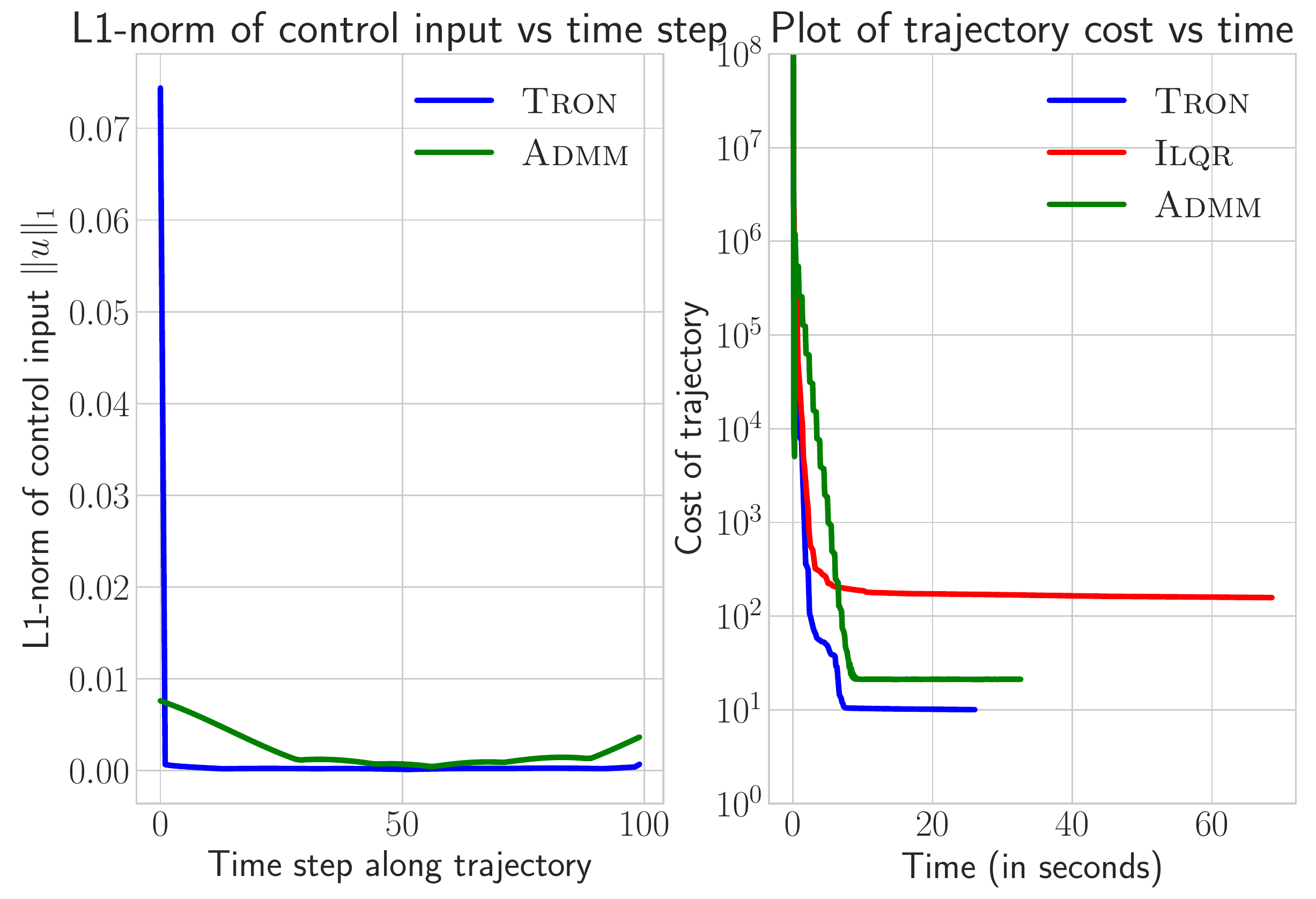}
  \caption{Performance of \tron{}, \admm{}, and \ilqr{} on the
    satellite rendezvous problem. On the right, we plot the cost of
    trajectory vs time (in seconds) for all the methods. On the left,
    we show L1-norm of the control input sequence in the final
    trajectory for \admm{} and \tron{}. We do not show \ilqr{} as it
    does not exhibit any sparse behavior. All approaches are run for
    $300$ iterations.}
  \label{fig:satellite}
\end{figure}

We compare \tron{} (with a fixed $\eta^k = 0.1$ for all iterations)
with \ilqr{} and \admm{}. The results are shown in
Figure~\ref{fig:satellite}. Interestingly, \ilqr{} exhibits fast
convergence initially until it reaches close to the minimum where the
non-smoothness of the objective results in extremely small updates and
slow convergence (See Figure~\ref{fig:satellite} right). \admm{} also
has extremely fast convergence at the
start and then increases the cost for a few iterations before
converging to a lower cost compared to ilqr{}. However, \tron{}
exhibits the fastest convergence among all and quickly reaches a
significantly lower cost. In Figure~\ref{fig:satellite} (left), we
plot the L1-norm of the control sequence for the final trajectory in
the case of \admm{} and \tron{}. We refrain from plotting the L1-norm
for \ilqr{} as its final trajectory does not exhibit any sparse
behavior and thus, is very undesirable for RCS control. As shown in
the plot, \tron{} converges to a final trajectory that uses the
thrusters at the beginning and then simply coasts for the rest of the
trajectory without using the thrusters. This behavior is ideal for RCS
thrusters. \admm, on the other hand, converges to a trajectory that
exhibits significantly less sparsity and thus, has higher cost.

\section{Extensions and Conclusion}
\label{sec:conclusion}

The proposed solver \tron{} can be extended in several ways. Firstly,
we can account for arbitrary non-linear constraints $\alpha, \beta$ in
problem~\eqref{eq:7} by using augmented lagrangian techniques such as
the ones used in Plancher
et. al.~\cite{DBLP:conf/iros/PlancherMK17}. Secondly, \tron{} is
easily extensible to cost functions where there are more than two
functions involved in the $\max$ operator. In such a case, the dual
variables would lie in a higher dimensional simplex but we still
achieve fast convergence since we are exploiting the
structure. Finally, \tron{} can be made numerically more robust by
employing techniques proposed in Howell et. al.~\cite{altro} such as
using square-root backward pass in \textsc{Ilqr}.

In conclusion, this work has proposed a fast solver \tron{} that can
be used as a general purpose tool in trajectory optimization where the
cost functions are non-differentiable with differentiable components. \tron{} exhibits
fast convergence behavior because it exploits the structure of the
cost function to construct a sequence of adaptively
smoothed objectives that can each be optimized efficiently. \tron{}
is provably guaranteed to converge to the global optimum in the case
of convex costs and linear dynamics, and to a stationary point in the case of
non-linear dynamics. Empirically, we show that \tron{} outperforms
other trajectory optimization approaches in simulated planning and
control tasks.

\section*{Acknowledgements}
\label{sec:acknowledgements}

The authors would like to thank Wen Sun for providing code and
pointing us to the surgical steerable needle application. In addition,
the authors would also like to thank the entire LairLab for insightful
discussion. AV would like to thank Jaskaran Singh, Ramkumar Natarajan,
Allison del Giorno, and
Anahita Mohseni-Kabir for reviewing the initial draft. AV is supported
by the CMU presidential fellowship endowed by TCS.

\bibliographystyle{IEEEtran}
\bibliography{bib}

\appendix

\subsection{Proof of Theorem~\ref{theorem:inexact-minimization}}
\label{sec:asympt-conv-under-1}

  Consider the quantity $\nabla_y \calL_{\eta^k}(y^k, \theta^{k-1})$
  for any iteration $k \geq 1$. Using equation~\ref{eq:6} we have,
  \begin{align*}
    &\nabla_y \calL_{\eta^k}(y^k, \theta^{k-1}) = \nabla_y f(y^k) + \\
    & \frac{\theta_1^{k-1}\exp(\frac{g_1(y^k)}{\eta^k})\nabla_y
      g_1(y^k) + \theta_2^{k-1}\exp(\frac{g_2(y^k)}{\eta^k})\nabla_y
      g_2(y^k)}{\theta_1^{k-1}\exp(\frac{g_1(y^k)}{\eta^k}) +
      \theta_2^{k-1}\exp(\frac{g_2(y^k)}{\eta^k})} \\
    &= \nabla_y f(y^k) + \\
    &\frac{\theta_1^{k-1}\exp(\frac{g_1(y^k)}{\eta^k})}{\theta_1^{k-1}\exp(\frac{g_1(y^k)}{\eta^k}) +
      \theta_2^{k-1}\exp(\frac{g_2(y^k)}{\eta^k})} \nabla_y g_1(y^k) +
    \\
    &\frac{\theta_2^{k-1}\exp(\frac{g_2(y^k)}{\eta^k})}{\theta_1^{k-1}\exp(\frac{g_1(y^k)}{\eta^k}) +
      \theta_2^{k-1}\exp(\frac{g_2(y^k)}{\eta^k})} \nabla_y g_2(y^k)
  \end{align*}
  Let us denote $\lambda_k = \frac{\theta_1^{k-1}\exp(\frac{g_1(y^k)}{\eta^k})}{\theta_1^{k-1}\exp(\frac{g_1(y^k)}{\eta^k}) +
      \theta_2^{k-1}\exp(\frac{g_2(y^k)}{\eta^k})}$, then it is easy
    to see that $1 - \lambda_k = \frac{\theta_2^{k-1}\exp(\frac{g_2(y^k)}{\eta^k})}{\theta_1^{k-1}\exp(\frac{g_1(y^k)}{\eta^k}) +
      \theta_2^{k-1}\exp(\frac{g_2(y^k)}{\eta^k})}$. Then we can
    rewrite, the above equation as,
    \begin{align}
      \label{eq:18}
      \nabla_y \calL_{\eta^k}(y^k, \theta^{k-1}) = \nabla_y f(y^k) &+
      \lambda_k\nabla_y g_1(y^k) +\nonumber\\
      &(1 - \lambda_k) \nabla_y g_2(y^k)
    \end{align}
    Take the limit of the above equation as $k \rightarrow \infty$,
    \begin{align}
      \label{eq:19}
      \lim_{k \rightarrow \infty} \nabla_y \calL_{\eta^k}(y^k, \theta^{k-1}) = \nabla_y f(y^*) &+
      \lim_{k \rightarrow \infty}\lambda_k\nabla_y g_1(y^*) +\nonumber\\
      &(1 - \lim_{k \rightarrow \infty}\lambda_k) \nabla_y g_2(y^*)
    \end{align}
   The hypothesis implies that $\lim_{k \rightarrow \infty} \nabla_y
   \calL_{\eta^k}(y^k, \theta^{k-1}) \rightarrow 0$. Thus all that
   remains to show is that $\lim_{k \rightarrow \infty}\lambda_k$ is
   finite. This is easy to prove. Note that for the limit point $y^*$
   there are one of three possibilities: $g_1(y^*) > g_2(y^*)$,
   $g_1(y^*) < g_2(y^*)$, or $g_1(y^*) = g_2(y^*)$. We will show the
   argument for one of these possibilities and the other two are very
   similar. Assume $g_1(y^*) > g_2(y^*)$ then we have that
   \begin{align*}
     \lim_{k \rightarrow \infty} \lambda_k &= \lim_{k \rightarrow \infty}
     \frac{\theta_1^{k-1}\exp(\frac{g_1(y^k)}{\eta^k})}{\theta_1^{k-1}\exp(\frac{g_1(y^k)}{\eta^k})
     + \theta_2^{k-1}\exp(\frac{g_2(y^k)}{\eta^k})} \\
     &= \lim_{k \rightarrow \infty} \frac{1}{1 +
       \frac{\theta^{k-1}_2}{\theta_1^{k-1}}\exp(\frac{g_2(y^k) -
       g_1(y^k)}{\eta^k})} \\
     &= \frac{1}{1 + \lim_{k \rightarrow \infty}
       \frac{\theta_2^{k-1}}{\theta_1^{k-1}} \exp(\frac{g_2(y^*) -
       g_1(y^*)}{\lim_{k \rightarrow \infty} \eta^k})} \\
     &= 1
   \end{align*}
   The last equality is obtained using the fact that $g_1(y^*) >
   g_2(y^*)$ and $\eta^k \rightarrow 0$. Similarly, we can prove that
   $\lim_{k \rightarrow \infty} \lambda_k = 0$ when $g_1(y^*) <
   g_2(y^*)$, and $\lim_{k \rightarrow \infty} \lambda_k = \lim_{k
     \rightarrow \infty} \frac{\theta_1^{k-1}}{\theta^{k-1}_1 +
     \theta_2^{k-1}} = \theta_{1}^{k-1}$ when $g_1(y^*) = g_2(y^*)$. Note that the
   sequence $\{\theta^k\}$ is bounded as they all lie in a simplex
   $\Delta_2$, thus $\lim_{k
     \rightarrow \infty} \theta_1^{k-1}$ is finite and lies in $[0, 1]$. Hence, we have
   that there exists some $\lambda \in [0, 1]$ such that for every
   limit point $y^*$ of the sequence $\{y^k\}$ satisfying the
   assumptions in the theorem,
   \begin{equation}
     \label{eq:20}
     \nabla_y f(y^*) + \lambda\nabla_y g_1(y^*) + (1 - \lambda)
     \nabla_y g_2(y^*) = 0
   \end{equation}
   \QEDA

\end{document}